\documentclass[letterpaper, 10 pt, conference]{ieeeconf}  

\IEEEoverridecommandlockouts                              

\usepackage[T1]{fontenc}
\usepackage{cite}
\usepackage{amsmath,amssymb,amsfonts}
\usepackage{algorithmic}
\usepackage{graphicx}
\usepackage{subcaption}
\usepackage{textcomp}
\usepackage{xcolor}
\usepackage{hyperref} 
\hypersetup{colorlinks, allcolors=blue}
\newtheorem{theorem}{Theorem}
\newtheorem{definition}{Definition}
\newtheorem{proposition}{Proposition}
\DeclareMathOperator*{\argmin}{arg\,min} 
 
\usepackage{multirow}
\usepackage{float}
\usepackage{placeins}
\usepackage{adjustbox}
\usepackage{epstopdf}

\title{\LARGE \bf
Signal Temporal Logic Planning with Time-Varying Robustness
}

\author{Yating Yuan, Thanin Quartz, Jun Liu 
\thanks{Yating Yuan, Thanin Quartz and Jun Liu are with Department of Applied Mathematics, 
        University of Waterloo, 200 University Avenue, Waterloo, ON, Canada N2L 3G1
        (email:
        {\tt\small yating.yuan, btquartz, j.liu@uwaterloo.ca)}.}%
}

\begin{document}

\maketitle
\thispagestyle{empty}
\pagestyle{empty}

\begin{abstract}

This letter aims to generate a continuous-time trajectory consisting of piecewise Bézier curves that satisfy signal temporal logic (STL) specifications with piecewise time-varying robustness. Our time-varying robustness is less conservative than the real-valued robustness, which enables more effective tracking in practical applications. Specifically, our continuous-time trajectories account for dynamic feasibility, leading to smaller tracking errors and ensuring that the STL specifications can be met by the tracking trajectory. Comparative experiments demonstrate the efficiency and effectiveness of the proposed approach.

\end{abstract}

\begin{keywords}
Signal temporal logic, trajectory planning, time-varying robustness measure, Bézier curves
\end{keywords}

\section{INTRODUCTION}
\label{sec:introduction}

In the rapidly evolving field of robotics, autonomous systems are expected to accomplish complex tasks. To meet these expectations, signal temporal logic (STL) offers a mathematically precise language for defining tasks and rules over continuous signals with explicit time semantics  \cite{pant2018fly, kurtz2020trajectory}. Furthermore, the STL robustness measure \cite{fainekos2009robustness} is widely utilized to quantify the degree to which a system's behavior satisfies the specified STL formula. However, this measure of robustness is time-invariant and underapproximates the actual degree by which the signal satisfies the STL specification over the entire time horizon. This may lead to two potential issues: 1) exceeding the robustness measure at specific points does not necessarily indicate a violation of the STL specification (as illustrated by tracking the first corner in Fig.~\ref{fig1(a)}); and 2) it may compromises the controllers' task performance by imposing unnecessarily strict constraints. For example, the controller may force a vehicle to slow down excessively or take overly cautious actions to meet tight error bounds. From this perspective, a time-varying robustness measure can mitigate these issues, offering a more adaptable and flexible approach for the actuators.

Moreover, digital systems have to approximate continuous-time robustness measures with discrete-time measures, making sampling time critical in practice. Coarse discretization in trajectory planning can lead to the trajectory passing through obstacles (as shown in Fig.~\ref{fig1(b)}). To address this issue, \cite{pant2018fly} adopts stricter discrete-time specifications, though this complicates the solution process. Meanwhile, \cite{yang2020continuous} utilizes control barrier functions to ensure STL satisfaction between timed waypoints. Fine discretization improves accuracy but increases the time complexity exponentially in mixed-integer convex programming (MICP) \cite{raman2014model, belta2019formal, kurtz2022mixed}. Solutions to this issue include gradient-based techniques with smoothing \cite{pant2017smooth, mehdipour2019arithmetic, gilpin2020smooth}, which may lack robustness guarantees, and methods reducing binary variables \cite{sun2022multi, kurtz2022mixed}, although these might overlook dynamic feasibility. The authors of \cite{kurtz2023temporal} used a graph of convex sets (GCS) to generate Bézier curves for motion planning but did not consider the robustness of these curves.

\textit{Contribution:} This letter features two main contributions: 1) we extend uniform robustness to time-varying robustness for STL specifications to avoid underestimating robustness and to enhance controller flexibility; 2) we establish theoretical guarantees on acceleration constraints at the control points to ensure that Bézier curves satisfy STL specifications with time-varying robustness, which can avoid violations caused by incorrectly chosen control points (as shown in Fig.~\ref{fig1(c)}). Additionally, our approach (i) is model-free, unlike the method based on model predictive control (MPC) \cite{yang2024signal} and the standard MICP approach \cite{kurtz2022mixed}; (ii) is dynamically feasible with smaller tracking errors compared to methods with unfixed sampling times \cite{sun2022multi}; (iii) addresses the issue of a trajectory crossing obstacles; and (iv) performs effectively in long time horizon scenarios.

\begin{figure}[!t]
    \centering
\parbox{3.2in}{%
\subfloat[\label{fig1(a)}]{%
\includegraphics[width=0.33\linewidth]{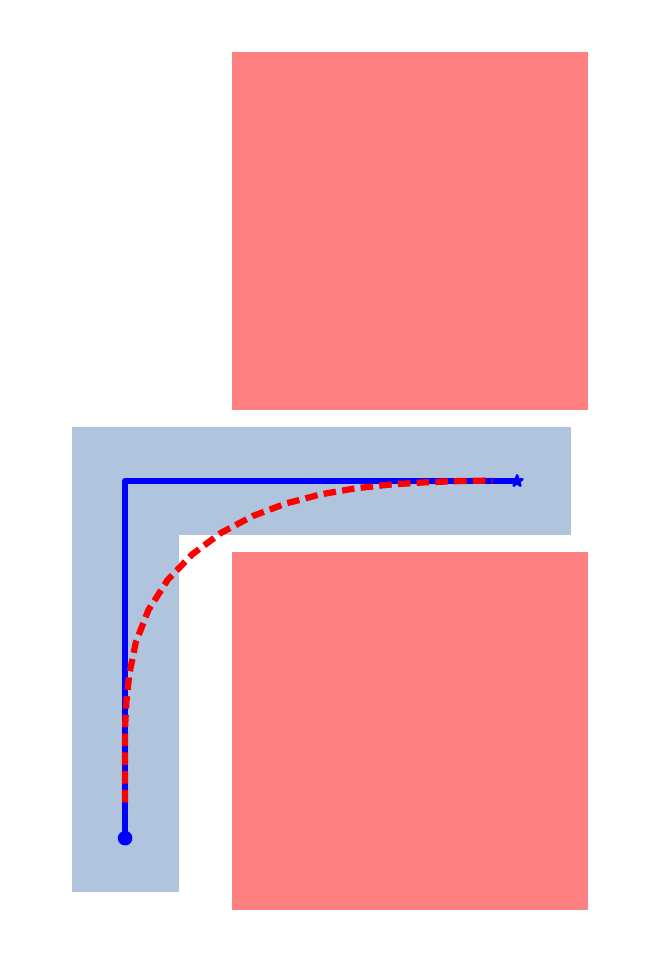}}
  \hfill
  \subfloat[\label{fig1(b)}]{%
\includegraphics[width=0.33\linewidth]{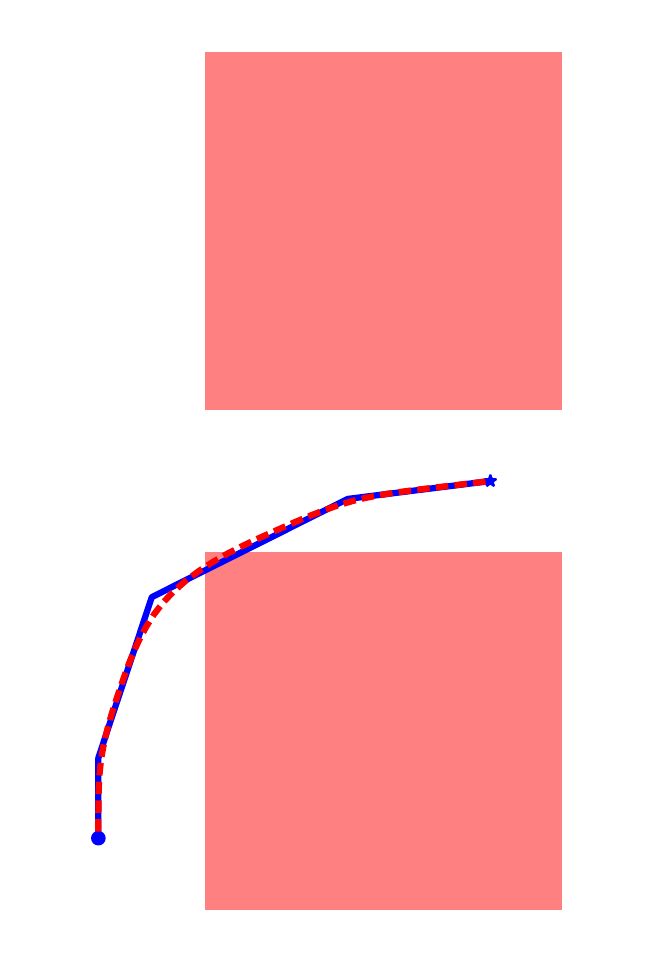}}
   \hfill
  \subfloat[\label{fig1(c)}]{%
\includegraphics[width=0.33\linewidth]{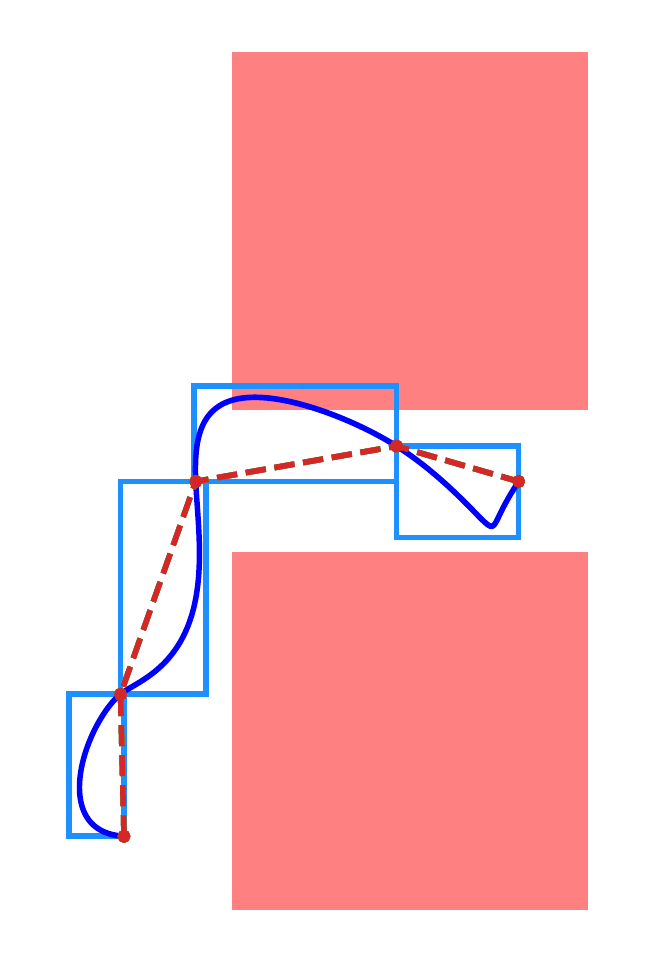}}
  }
\caption{A simple reach-avoid task: (a) presents a reference trajectory (blue solid line) with a real-valued robustness tube (blue shadow) and its tracking result (red dashed line); (b) shows a trajectory intersecting obstacles between time steps; (c) illustrates that the Bézier curves (blue solid line) violate the specifications due to incorrectly chosen control points. The red dashed lines are linear segments between the end points of the Bézier curves and the blue frames are the bounding boxes of the curves.}
\label{fig.conservative_robustness}
\end{figure}

\textit{Notation:} 
Let $\mathbb{W} \subset \mathbb{R}^d$ represent the workspace. For a vector $x \in \mathbb{W}$, $|x|$ denotes the element-wise absolute value of $x$ and $\|x \|$ denote the Euclidean norm. Let $T>0$ be the time duration of a signal, with $\mathbb{T}=[0, T] \subset \mathbb{R}_{\geq 0}$ representing the bounded time domain. Given a signal $y: \mathbb{T} \rightarrow \mathbb{W}$, let $\left(y,\left[t_i, t_j\right]\right) = \{y(t) \mid  t \in \left[t_i, t_j\right], 0 \leq t_i < t_j \leq T\}$ denote the portion of $y$ within the interval $\left[t_i, t_j\right]$, where we also denote $(y, t) := \left(y,\left[t, T\right]\right)$ for brevity. The distance between the signal $y$ and a subset $\mathbb{Y} \subset \mathbb{W}$ is denoted by $\operatorname{dist}(y, \mathbb{Y}):= \inf_{t \in \mathbb{T}}\inf_{a \in \partial \mathbb{Y}}\|y(t)-a\|_2$, where $\partial \mathbb{Y}$ is the boundary of $\mathbb{Y}$. Let $\text{Poly}(H, b) := \left\{x \in \mathbb{R}^d \mid H x \leq b\right\} \subseteq \mathbb{W}$ represent a convex polytope, where $H \in \mathbb{R}^{r \times d}$ and $b \in \mathbb{R}^r$. We use $H_{i}$ to denote the $i$-th row of $H$ and $b_{i}$ the $i$-th element of the vector $b$. We also denote $\operatorname{Row}(H)$ as the number of rows in $H$, which corresponds to the number of faces of the polytope.

\section{Preliminary}
\label{sec:Preliminary}

\subsection{Bézier Curves}
In this letter, a continuous-time position trajectory $y(t):[0, T] \rightarrow \mathbb{W}$ is parameterized by $N$ piecewise Bézier curves. Specifically, with a sampling time $\Delta t=\frac{T}{N}$, the trajectory $y$ is defined as follows:
\begin{equation}
y(t)= 
\begin{cases}
B_0(t), & t \in\left[t_0, t_1\right], \\ 
B_1(t), & t \in\left[t_1, t_2\right], \\ 
\vdots & \\ 
B_{N-1}(t), & t \in\left[t_{N-1}, t_{N}\right],
\end{cases}\\ 
\label{eq.y(t)}
\end{equation}
and each Bézier curve is $n$-degree and described as 
\begin{equation}
B_k(t) = \sum_{i=0}^n b_i^n(\tau) c_{k,i}, t \in [t_k, t_{k+1}],
\end{equation}
where $c_{k, i}$ is the $i$-th control point of the $k$-th Bézier curve, $t_k=k \Delta t$ for $k=\{0,1, \ldots, N - 1\}$, $\tau=\frac{t-t_k}{t_{k+1}-t_k}$, and $b_i^n(\tau)=\frac{n!}{(n-i)!i!}(1-\tau)^{n-i} \tau^i$ is the $n$-degree Bernstein polynomial. Since $ \sum_{i=0}^n b_i^n(\tau)=1$ for any $\tau \in[0,1]$, this implies that the convex hull of the set of control points, denoted by $\mathcal{CH}_k:=\operatorname{conv}\left\{c_{k,0}, c_{k,1}, \ldots, c_{k,n}\right\}$, contains all points on $B_k$. 

\subsection{Signal Temporal Logic}

 Given a vector-valued function $\mu$ defined on $\mathbb{W}$, an atomic predicate $\pi$ can be defined based on $\mu$. In this letter, we only consider atomic predicates of the form $\pi:= \mu(x) = b - Hx \geq 0$ which determine whether a point $x$ lies inside the convex polytope $\text{Poly}(H, b)$. In particular, $x \in \operatorname{Poly}(H,b)$ if and only if $\mu(x) \geq 0$. The syntax of STL is defined as:
\begin{equation}
\varphi:= \pi \mid \neg \varphi \mid  \varphi_1 \wedge \varphi_2 \mid \varphi_1 \lor \varphi_2 \mid \varphi_1 \mathcal{U}_{I} \varphi_2,
\end{equation}
where $\neg$, $\wedge$ and $\lor$ are negation, conjunction and disjunction operators, respectively; $I=[a, b] \subseteq \mathbb{T}$ is a bounded time interval and we assume $I$ is not a singleton, i.e., $a<b$; and $\mathcal{U}_I$ is \textit{Until}, which means $\varphi_2$ will be true at some point within $I$, \textit{until} that moment, and $\varphi_1$ has to remain true. Other temporal operators, \textit{Eventually} ($\Diamond$) and \textit{Always} ($\square$) can be formulated by the above operators. For instance, $\Diamond_{I} \varphi :=\text { True } \mathbf{U}_{I}$ and $\square_{I} \varphi:=\neg\left(\Diamond_{I} \neg \varphi\right)$. In this letter, all STL formulas are in negation normal form (NNF)\cite{baier2008principles}, also referred to as positive normal form (PNF) \cite{darwiche2001decomposable}, where all negations are applied only to propositions. This is not restrictive since any STL formula can be written in this form \cite{LaValle2006-oe}. This is formalized in the following definition.

\begin{definition}{(STL semantics \cite{donze2010robust})} 
\label{def. continuous semantic}
Given an STL specification $\varphi$ and a continuous-time signal $y: \mathbb{T} \rightarrow \mathbb{W}$, STL semantics is defined over suffixes of signals as follows: 
$$
\begin{array}{lll}
(y, t) \models \pi & \Leftrightarrow & \mu(y(t)) > 0; \\ 
(y, t) \models \neg \varphi & \Leftrightarrow & \neg ((y, t) \models \varphi); \\
(y, t) \models \varphi_1 \wedge \varphi_2 & \Leftrightarrow & (y, t) \models \varphi_1 \wedge(y, t) \models \varphi_2; \\
(y, t) \models \varphi_1 \vee \varphi_2 & \Leftrightarrow & (y, t) \models \varphi_1 \vee(y, t) \models \varphi_2 ;\\
(y, t) \models \Diamond_{I}  \varphi & \Leftrightarrow & \exists t^{\prime} \in \tilde{I},\left(y, t^{\prime}\right) \models \varphi; \\
(y, t) \models \square_I \varphi & \Leftrightarrow & \forall t^{\prime} \in \tilde{I},\left(y, t^{\prime}\right) \models \varphi;\\
(y, t) \models \varphi_1 \mathcal{U}_{[a, b]} \varphi_2 & \Leftrightarrow & \exists t^{\prime} \in \tilde{I} \left(y, t^{\prime}\right) \models \varphi_2 \\
& & \wedge \forall t^{\prime \prime} \in\left[t, t^{\prime}\right],\left(y, t^{\prime \prime}\right) \models \varphi_1,\\
\end{array}
$$
where $\tilde{I} = (t+I) \cap \mathbb{T}$ and $ t+I=\left[t+a, t+b\right]$. We write $(y, t) \models \varphi$ if the signal suffix $(y, t)$ satisfies the STL formula $\varphi$ and $(y, t) \nvDash \varphi$ if the signal suffix $(y, t)$ does not satisfy $\varphi$. For brevity, note that $y \models \varphi$ if $(y, 0) \models \varphi$.
\end{definition}

To measure the degree of robustness with respect to an STL specification $\varphi$ over varying times, we propose the following definition of piecewise time-varying robustness.

\begin{definition} (Time-varying robustness $\rho^\varphi(t)$)
\label{def:Time-varying Robustness}
Given a function $\rho^\varphi: \mathbb{T} \to \mathbb{R}^+$, a position trajectory $y : \mathbb{T} \rightarrow \mathbb{W}$ is said to be $\rho^\varphi$-robust with respect to an STL specification $\varphi$ if $\hat{y} \models \varphi$  for all trajectories $\hat{y}$ satisfying
\begin{equation}
\left\|\hat{y}(t) - y(t)\right\|_2 \leq \rho^\varphi(t),\quad \forall t \in \mathbb{T}.  
\end{equation}
\end{definition}

\section{Problem Statement}
\label{sec:PS}


The formal definition of our planning problem is provided as follows. 

\begin{definition}(STL planning problem based on Bézier curves)
\label{def.planning} 
Given an STL specification $\varphi$, the problem is to find a position trajectory $y: \mathbb{T} \rightarrow \mathbb{W}$, parameterized by $N$ piecewise $n$-degree Bézier curves as in \eqref{eq.y(t)}, that satisfy the following conditions:
\begin{enumerate}
    \item \textit{Continuity conditions}: The Bézier curves have to be at least $C^2$-continuous; \label{item: CC}
     \item \textit{Dynamic conditions}: Given the maximum velocity and acceleration $v_{\max }, a_{\max} \in \mathbb{R}^d$, the conditions $|B_k^{\prime}(t)| \leq v_{\max}$ and $|B_k^{\prime \prime}(t)|\leq a_{\max}$ hold for all $k \in [0, \dots, N-1]$; \label{item: DC}
    \item \textit{STL Satisfaction}: Consider an STL specification $\varphi$ and a desired minimum robustness $\rho^*$. We require that each Bézier curve $B_k$ be $\rho^\varphi$-robust, with its robustness given by $\rho^{\varphi}(t) = \rho_k^{\varphi}(\mathcal{C}_k, a_k) \geq \rho^* > 0$ for all $t \in [t_k, t_{k+1}]$. Here, $\rho_k^{\varphi}(\mathcal{C}_k, a_k)$ is a linear robustness measure of the Bézier curve, which depends on the maximum acceleration $a_k \in \mathbb{R}^d$ and the set of control points of $B_k$, denoted by $\mathcal{C}_k:= \bigcup_{i=0}^{n}\{c_{k,i}\}$. \label{item: STL_C}
\end{enumerate}
\label{def.Bezier curves-based STL problem}
\end{definition}

We formulate the problem of finding a trajectory satisfying the conditions specified in Definition \ref{def.planning} as a constrained optimization problem which maximizes the time-varying robustness and minimizes the control effort. We derive constraints for points \ref{item: CC})-\ref{item: DC}) in Section \ref{Section_traj} and for point \ref{item: STL_C}) in Section \ref{Section_stlspec}. This optimization problem is encoded as a mixed-integer convex program (MICP) as follows: 

\begin{equation}
\begin{aligned}
\argmin_{\mathcal{C}_k, v_k, a_k} & \quad \sum_{k=0}^{N-1}- \lambda\rho_k^{\varphi}(\mathcal{C}_k, a_k)+Q\left\|v_k\right\|_1+R\left\|a_k\right\|_1 \\
\text { s.t. } & \quad 
B_k \; \text{satisfies the conditions in Definition \ref{def.Bezier curves-based STL problem}}, \\
& \quad \forall k\in\{0, \dots, N-1\}, 
\end{aligned}
\end{equation}
where $\lambda, Q, R >0$ are the weights, and $v_k, a_k \in \mathbb{R}^d$ are the maximum velocity and acceleration of $B_k$, respectively. In Section \ref{subset: Time-Varying Robustness for Bézier Curves}, we prove that our STL encoding satisfies the soundness property. Consequently, this program solves for Bézier curves that satisfy the STL specification $\varphi$ in one shot.

\section{Trajectory Construction} \label{Section_traj}

In this section, we derive linear constraints on the control points to ensure that the Bézier curves satisfy the continuity and dynamic conditions outlined in Definition \ref{def.planning}.

\subsection{Continuity Constraints}
\begin{itemize}
    \item \textit{$C^0$-Continuity:} To ensure $B_k(t) = B_{k+1}(t)$ at $t = t_{k+1}$, we impose the following constraint:
    \begin{equation}
        c_{k, n} = c_{k+1, 0}.
        \label{eq.constr_C^0}
    \end{equation} 
    
    \item \textit{$C^1$-Continuity:} For $B_k^{\prime}(t) = B_{k+1}^{\prime}(t)$ at $t = t_{k+1}$, the following constraints are enforced: 
\begin{equation}
c_{k, n}-c_{k,n-1} = c_{k+1, 1}-c_{k+1,0}.   
\label{eq.constr_C^1}   
\end{equation}

\item \textit{$C^2$-Continuity:} For $B_k^{\prime \prime}(t) = B_{k+1}^{\prime\prime}(t)$ at $t = t_{k+1}$, let $\Delta^2 c_{k,i} := c_{k,i+2} - 2c_{k,i+1} + c_{k,i}$. The following constraint is then enforced: 
\begin{equation}
\Delta^2 c_{k,n-2} = \Delta^2 c_{k+1,0}.
 \label{eq.constr_C^2}
\end{equation}
\end{itemize}


\subsection{Dynamic Constraints}
The Bézier curves are required to satisfy the dynamic constraints \( |\dot{y}(t)| \leq v_{\max} \) and \( |\ddot{y}(t)| \leq a_{\max} \), where \( v_{\max}, a_{\max} \in \mathbb{R}^d \) represent the maximum velocity and acceleration over the whole time horizon, respectively.

\begin{itemize}
\item \textit{Velocity Constraints:} For $t \in [t_k, t_{k+1}]$, it follows that
\begin{equation*}
\begin{aligned}
|\dot{y}(t)| &= |B_k^{\prime}(t)| \\
&\leq \frac{n}{\Delta t}\left|\sum_{i=0}^{n-1}b_i^{n-1}(\tau) \right| \max_i |c_{k, i+1}-c_{k,i}| \\
& \leq v_{k}, \\
\end{aligned}
\end{equation*}
Since $\sum_{i=0}^{n-1} b_i^{n-1}(\tau)=1$ for any $\tau \in[0,1]$, the velocity constraints of Bézier curves are transformed into the following linear constraints:
\begin{align}
    |c_{k, i+1}-c_{k,i}| &\leq \frac{v_{k}\Delta t}{n}, \forall i \in \{0, \ldots, n-1\},  \label{eq.|cp|<v_k}\\
    0 < v_{k} &\leq v_{\max}.
    \label{eq.v_k<v_max}
\end{align}

\item \textit{Acceleration Constraints:}
Similarly, the acceleration constraints of the Bézier curves are enforced as
\begin{align}
    |\Delta^2 c_{k,i}| &\leq \frac{a_{k}\Delta t^2}{n(n-1)}, \label{eq.|cp|<a_k}\\
    0 < a_{k} &\leq a_{max}.
    \label{eq.a_k<a_max}
\end{align}
\end{itemize}

\section{Encoding STL Satisfactions with Bézier Curves}

\label{Section_stlspec}
In this section, we discuss how to encode an STL specification $\varphi$ using linear constraints such that the Bézier curves can satisfy the STL specification with robustness $\rho^{\varphi}(t)$. 

\subsection{Encoding STL Satisfactions with Linear Constraints}
We introduce the binary variable $z_k^\varphi$ to denote if the Bézier curve segment $B_k$ meets the linear constraint, ensuring $z_k^\varphi$ is sound. Specifically, if $z_k^{\varphi}$ is true, then $\forall t \in [t_k, t_{k+1}], (B_k, t) \vDash \varphi$. As stated in Section \ref{sec:PS}, we set the robustness measure of $B_k$ to be $\rho_k^{\varphi}(\mathcal{C}_k, a_k) = r_k - \epsilon_k \ge  \rho^*$ where $r_k, \epsilon_k \in \mathbb{R}_{>0}$, both of which are defined in the STL predicate below. 
\begin{itemize}
    \item Predicate $\pi$: 
    \begin{equation} 
    \begin{aligned}
    z_k^\pi =  &\bigwedge_{i=1}^{\operatorname{Row}(H)} \left(\bigwedge_{j=\{0,n\}} \frac{b_i-H_ic_{k, j}}{\left\|H_i\right\|_2}- r_k \geq  0 \right)\\
    &\quad \bigwedge \left( \bigwedge_{j=\{1,n\}} |c_{k,j} - c_{k, j-1}| \leq \frac{a_k \Delta t^2}{2n} \right)\\
    &\quad \bigwedge \left(\bigwedge^d_{j=0}\frac{8\epsilon_k}{3\sqrt{d} \Delta t^2}-|a_{k,j}| \geq 0 \right),\\ 
    \label{eq.predicate}
    \end{aligned}
    \end{equation}
    where $a_{k,i}$ is the $i$-th dimension of $a_k$, $d$ is the dimension of $a_k$. 
    \item Negation $\neg \pi$: 
    \begin{equation}
\begin{aligned}
    z_k^{\neg \pi} =  &\bigvee_{i=1}^{\operatorname{Row}(H)} \left(\bigwedge_{j=\{0,n\}}  \frac{H_i c_{k, j} - b_i}{\left\| H_i \right\|_2} - r_k \geq 0  \right) \\
    & \quad \bigwedge \left(\bigwedge_{j=\{1,n\}} |c_{k,j} - c_{k,j-1}| \leq \frac{a_k \Delta t^2}{2n} \right) \\
    &  \quad \bigwedge \left( \bigwedge^d_{j=0} \frac{8\epsilon_k}{3\sqrt{d} \Delta t^2} - |a_{k,j}| \geq 0 \right),
\end{aligned}
\label{eq.negation}
\end{equation}
Based on the atmoic predicates, the logic and temporal operators are constructed as follows. Because of the finite interval $\tilde{I}=(t+I) \cap \mathbb{T}$, we note $\ell_l =k+\lfloor a/\Delta t\rfloor$ and $\ell_r=\min(k+1+\lfloor b/\Delta t\rfloor,N-1)$ for brevity.
    \item Conjunctions $\varphi = \varphi_1 \wedge \varphi_2$: \begin{equation}
        z_k^{\varphi} = z_k^{\varphi_1} \wedge z_k^{\varphi_2}. 
        \label{eq.conj}
        \end{equation}
    
    \item Disjunctions $\varphi = \varphi_1 \vee \varphi_2$: 
    \begin{equation}z_k^\varphi = z_k^{\varphi_1} \vee z_k^{\varphi_2}. 
    \label{eq.disjun}
    \end{equation}

    \item Always $\square_I \varphi$: \begin{equation}
    z_k^{\square_{I} \varphi} = \wedge_{j=\ell_l}^{\ell_r}z_j^\varphi. 
    \label{eq.always}
    \end{equation}

    \item Eventually $\Diamond_{\Tilde{I}} \varphi$: 
    \begin{equation}
    z_k^{\Diamond_{\Tilde{I}} \varphi} = \vee_{j=\ell_l}^{\ell_r}z_j^\varphi. 
    \label{eq.eventually}
    \end{equation}

    \item Until $\varphi=\varphi_1 \mathcal{U}_{\Tilde{I}} \varphi_2$: 
\begin{equation}
z_k^{\varphi} = \Big(\vee_{k^{\prime}=\ell_l}^{\ell_r} z^{\varphi_2}_{k^{\prime}} \Big) \bigwedge \Big(\wedge_{k^{\prime\prime}=k}^{k^{\prime}-1}z^{\varphi_1}_{k^{\prime\prime}}\Big).
\label{eq.Until}
\end{equation}
\end{itemize}

To force $z_k^{\varphi}=1$, we utilize the Big-M method. As \cite{sun2022multi, kurtz2022mixed}, the binary variables $z_i$ and a large enough positive constant $M$ are only introduced when converting disjunctions into conjunctions. Specifically, given the disjunctions with the linear constraints $\mathrm{LC}_i$ as $\bigvee_{i=1}^n \mathrm{LC}_i \geq 0$, the equivalent conjunctions are $\left(\bigwedge^n_i \mathrm{LC}_i+\left(1-z_i\right) M \geq 0 \wedge \sum_{i=1}^n z_i \geq 1\right)$. Thus, when $z_i=1$, the constraint $\mathrm{LC}_i \geq 0$ holds; otherwise, $\mathrm{LC}_i+M \geq 0$ is always true regardless of the value of $\mathrm{LC}_i$ if $z_i=0$. At least one of the constraints is satisfied by enforcing $\sum_{i=1}^n z_i \geq 1$. Ultimately, all conjunctions are encoded without introducing binary variables, strictly enforcing the linear constraints $\mathrm{LC}_i \geq 0$.

\subsection{The Soundness of STL Encoding}
\label{subset: Time-Varying Robustness for Bézier Curves}
We first show that the $z_k^\pi$ and $z_k^{\neg \pi}$ satisfy the soundness property. Moreover, a solution satisfying these constraints will also yield the robustness measure of the position trajectory $y$. The end control points, namely $B_k(t_k)=c_{k,0}$ and $B_k(t_{k+1})=c_{k,n}$, act as anchors $\forall t \in\left[t_k, t_{k+1}\right]$. The idea of the proof is to use the satisfaction of \eqref{eq.predicate} or \eqref{eq.negation} to argue that the end control points are at least a distance $r_k$ from $\operatorname{Poly}(H,b)$ and that the Bézier curve is at most a distance $\epsilon_k$ from the end control points. This is formally shown in the following proposition.

\begin{proposition}
For $t\in [t_k, t_{k+1}]$, if an $n$-degree Bézier curve $B_k(t)$ satisfies \eqref{eq.predicate}, then $B_k \subset \text{Poly}(H,b)$. Similarly, if $B_k$ satisfies \eqref{eq.negation}, then $B_k \cap \text{Poly}(H,b) = \emptyset$.
\label{Prop.B_k_in_Poly}
\end{proposition}

\begin{proof}
We first need to bound the distance from the other control points to the corresponding end control point. Note that from (\ref{eq.predicate}) we have that $\|c_{k, 1} - c_{k, 0} \| \leq \frac{\| a_k \| \Delta t^2}{2n}$ and $\|c_{k, n-1} - c_{k, n} \| \leq \frac{\| a_k \| \Delta t^2}{2n}$. The acceleration constraint \eqref{eq.|cp|<a_k} on the control points implies that $\| c_{k, i+2} - 2 c_{k, i + 1} + c_{k, i} \| \leq \frac{\|a_k\| \Delta t^2}{n (n - 1)}$. Now suppose that $2 \leq i \leq \lfloor \frac{n}{2} \rfloor$ as the case when $\lfloor \frac{n}{2} \rfloor + 1 \leq i \leq n - 2$ is similar. By repeated application of the triangle inequality we see that 
\begin{align*}
    \|c_{k, i} - c_{k, 0} \| \leq &\sum_{j = 0}^{i - 2} (j + 1) \| c_{k, i - j} - 2 c_{k, i - j - 1} + c_{k, i - j - 2} \| \\ 
    &+ i \|c_{k, 1} - c_{k, 0} \| \\
    \leq &\sum_{j = 0}^{i - 2} (j + 1) \frac{\|a_k \| \Delta t^2}{n(n-1)} + i \|c_{k, 1} - c_{k, 0} \| \\
    \leq &  \frac{i(i-1) \|a_k\| \Delta t^2}{2n(n-1)} +  i \frac{\| a_k \| \Delta t^2}{2n}
\end{align*}
 since $i \leq \frac{n}{2}$, this implies that $\|c_{k, i} - c_{k, 0} \| \leq \frac{3\Delta t^2 \|a_k \|}{8}$.
Therefore, for $2 \leq i \leq n - 2$, we have the following bound 
\begin{equation}
    \min \{\|c_{k, i} - c_{k, 0} \|, \|c_{k, i} - c_{k, n} \| \} \leq \frac{3\Delta t^2 \|a_k \|}{8}.    
\end{equation}

Now suppose that (\ref{eq.predicate}) holds. Denote the set obtained by shrinking $\text{Poly}(H, b)$ by some $c>0$ as 
\begin{equation*}
\mathrm{P}_c := \big\{q \in \mathcal{W} \mid \bigwedge_{i=1}^{\operatorname{Row}(H)} b_i-H_i \cdot q- c\left\|H_i\right\|_2 \geq 0\big\}.
\end{equation*}
Thus, the constraints (\ref{eq.predicate}) implies that $\frac{b_i -  H_i c_{k, 0}}{\left\|H_i\right\|_2} -r_k \geq 0$ and $\frac{b_i - H_i c_{k, n}}{\left\|H_i\right\|_2} -r_k \geq 0$ for all $i \in \{1, \hdots, \operatorname{Row}(H) \}$, which ensures both $c_{k, 0}$ and $c_{k,n}$ are in $P_{r_k}$. As $\frac{8\epsilon_k}{3 \sqrt{d} \Delta t^2}-|a_{k,i}| \geq 0$ holds for all $i \in \{1, \hdots, d \}$, this implies that $\frac{3\Delta t^2 \|a_k\|}{8} \leq \epsilon_k$. Since we force $\rho_k^{\varphi}(\mathcal{C}_k, a_k) = r_k - \epsilon_k \ge  \rho^*$, combining the inequalities gives
\begin{equation}
    \min \{ \| c_{k, i} - c_{k, 0} \|, \| c_{k, i} - c_{k, n} \| \}) \leq \epsilon_k < r_k
\end{equation} for all $i \in \{1, \hdots, \operatorname{Row}(H) \}$. Given that the Bézier curve $B_k$ is contained in the convex hull of its control points, $B_k \subset \mathcal{CH}_k$,  we can conclude that $B_k \subset \mathcal{CH}_k \subset P_{r_k - \epsilon_k}$. This is because $\mathcal{CH}_k$ and $\operatorname{Poly}(H, b)$ are closed and  $\operatorname{dist}(\mathcal{CH}_k, \operatorname{Poly}(H, b))$ is attained at the extremal points which in this case are the control points. Therefore, $B_k \subset \text{Poly}(H, b)$. Note that $B_k \subset P_{r_k - \epsilon_k}$ is equivalent to the distance between $B_k$ and $\text{Poly}(H, b)$ being at least $r_k - \epsilon_k$.

For the second part of the proposition, suppose that (\ref{eq.negation}) holds and denote the half space
\begin{equation*}
\text {HS}^{i}_c :=\left\{q \in \mathcal{W} \mid H_i \cdot q-b_i-c\left\|H_i\right\|_2 \geq 0\right\}.
\end{equation*}
Since (\ref{eq.negation}) holds, there exists $j \in\{1, \cdots, \operatorname{Row}(H)\}$ such that $H_j \cdot c_{k, 0} -b_j-r_k\left\|H_j\right\|_2 \geq 0$ and $H_j \cdot c_{k, n} -b_j-r_k \left\|H_j\right\|_2 \geq 0$, which means  $c_{k,0}, c_{k,n} \in \text{HS}^{j}_{r_k}$. Analogously to the argument above we can show that $B_k \subset \mathcal{CH}_k \subset \operatorname{HS}^{j}_{r_k - \epsilon_k}$ and, therefore, it follows that $B_k \cap \text{Poly}(H,b) = \emptyset$.
\end{proof}


Furthermore, the subsequent result establishes the soundness property for both the logical and temporal operations applied to the atomic predicates.

\begin{theorem}

Consider an STL specification $\varphi$, a trajectory $y$ that constructed  by $N$ piecewise Bézier curves, and the constraints $z^{\varphi}_k$ for $k=\{0, \dots, N-1\}$, encoded with \eqref{eq.predicate}--\eqref{eq.Until}. For any position trajectory $p$ satisfying $\left\|p(t)-B_k(t)\right\|_2 \leq \rho^{\varphi}_k(\mathcal{C}_k, a_k)$ for all $t \in\left[t_k, t_{k+1}\right]$ and for all $k \in\{0, \ldots, N-1\}$, if $z_i^{\varphi}$ is true, then $\forall t \in\left[t_i, t_{i+1}\right],(p, t) \vDash \varphi$.
\end{theorem}

\begin{proof}
Proposition \ref{Prop.B_k_in_Poly} proves the result when $\varphi$ is an atomic predicate or its corresponding negation. Based on these two base cases, the inductive step proceeds as follows.

Conjunction: If $z_k^{\varphi}$ is true, then both $z_k^{\varphi_1}$ and $z_k^{\varphi_2}$ are true. This implies that $(p, t) \vDash \varphi_1 \wedge (p, t) \vDash \varphi_2 \Rightarrow (p, t) \vDash \varphi_1 \wedge \varphi_2$ for all $t \in\left[t_i, t_{i+1}\right]$.

Disjunction: If $z_k^{\varphi}$ is true, then at least one of $z_k^{\varphi_1}$ and $z_k^{\varphi_2}$ is true. This implies that $(p, t) \vDash \varphi_1 \vee (p, t) \vDash \varphi_2 \Rightarrow (p, t) \vDash \varphi_1 \vee \varphi_2$ for all $t \in\left[t_i, t_{i+1}\right]$.

Always: If $z_k^{\square_{[a, b]} \varphi}$ is true, then $z_j^{\varphi}$ is true for all $j \in [k+\lfloor a / d t\rfloor, \min(k+1+\lfloor b / d t\rfloor, N-1)]$. This means $(p,t') \models \varphi$ for all $t' \in [t_j, t_{j+1}] \subset [t_k+a, t_{k+1}+b] \cap [0, T]$. Therefore, for any $t \in [t_k, t_{k+1}]$, $(p,t') \models \varphi$ for $t' \in [t+a, t+b] \cap [0, T]$. 

\vspace*{2.mm}
Eventually: If $z_k^{\Diamond{[a, b]} \varphi}$ is true, then there exist at least one $j \in [k+\lfloor a / d t\rfloor, \min(k+1+\lfloor b / d t\rfloor, N-1)]$ such that $z_j^{\varphi}$ is true $\Rightarrow (p,t') \models \varphi$ for $\forall t' \in [t_j, t_{j+1}] \subset [t_k+a, t_{k+1}+b] \cap [0, T]$. This means for any $t \in [t_k, t_{k+1}]$, $\exists t' \in [t+a, t+b] \cap [0, T], (p,t') \models \varphi$.

Until: If $z_k^{\varphi_1 \mathcal{U}_{\tilde{I}} \varphi_2}$ is true, then for any $t \in\left[t_k, t_{k+1}\right]$, there exists $k^{\prime} \in [k+\lfloor a / d t\rfloor, \min(k+1+\lfloor b / d t\rfloor, N-1)]$ such that $z_{k^{\prime}}^{\varphi_2}$ is true. Moreover, for any $k^{\prime \prime} \in\left[k+1, k^{\prime}-1\right]$, $z_{k^{\prime \prime}}^{\varphi_1}$ is true. This implies that any $t \in [t_k, t_{k+1}]$, $\exists t^{\prime} \in [t+a, t+b] \cap [0, T], (p,t') \models \varphi_2$ and $\forall t^{\prime \prime} \in [t_{k}, t^{\prime}], (p,t') \models \varphi_1$. 
\end{proof}

As a final remark, we analyze the encoding complexity for a trajectory with $N$ Bézier curve segments. Since we only introduce binary variables for the disjunctions, let $|\varphi|$ denote the number of disjunction operators in the STL formula. Then, for each $n$-degree Bézier curve $B_k$, we require $\mathcal{O}(|\varphi|)$ binary and $\mathcal{O}(n+4)$ continuous variables, including $v_k$, $a_k$, $r_k$ and $\epsilon_k$. Since $n$ is typically less than 10 in this letter, it is appropriate to approximate $\mathcal{O}(n+4)$ as $\mathcal{O}(1)$. Consequently, the generalized numbers of binary and continuous variables are $\mathcal{O}(N|\varphi|)$ and $\mathcal{O}(N)$, respectively.

\section{Experiments}

In this section, we evaluate the performance of our algorithm on several benchmarks. The algorithm is implemented in Python, with the MICP problems solved using Gurobi \cite{gurobi}, and the tracking problem addressed with cvxpy \cite{cvxpy}. All experiments were conducted on an Apple M1 chip with 8 cores and 16 GB RAM. 

\subsection{Mission Scenarios}
We implement our algorithm on the following  three mission scenarios:   
\begin{itemize}
    \item (Reach-avoid) A robot must avoid two obstacles, the yellow (Y) and blue (B) regions, and visit the red (R) and green (G) regions for 2 seconds each. The STL formula is formulated as
    \begin{equation*}
    \begin{aligned}
    \varphi = &\left(\square_{[0, T]} \neg \text{B}\right) \wedge\left(\square_{[0, T]} \neg \text{Y}\right) \\
    &\wedge\left(\Diamond_{[0, T]} \square_{[0,2]} \text{R}\right) \wedge \left(\Diamond_{[0, T]} \square_{[0,2]} \text{G}\right).
    \end{aligned}
    \end{equation*}

    \item (Complex reach-avoid) A robot must avoid multiple obstacles and visit numerous charging stations. The specification $\varphi$ can be formally expressed as
    \begin{equation*}
    \begin{aligned}
    \varphi = & \square_{[0,T]} \left(\wedge_i^4 \neg \text{Obs}_i \right) \wedge  \Diamond_{[0,T]}\left(\square_{\left[0, 3 \right]} \left(\vee_i^4 \text{Chg}_{1,i} \right) \right) \\ 
    & \wedge \Diamond_{[0,T]}\left(\square_{\left[0, 3\right]}\left(\vee_i^3\text{Chg}_{2,i}\right)\right). 
    \end{aligned} 
    \end{equation*}
    
    \item (Narrow passing) 
    The scenario involves narrow passages. A robot stays at one of two possible charging stations for 2 seconds while avoiding obstacles and achieving the goal. This specification $\varphi$ can be formally written as
    \begin{equation*}
    \begin{aligned} 
    \varphi = & \Diamond_{[0,T]}\left(\square_{\left[0, 3 \right]} \left(\vee_i^2 \text{Chg}_{i} \right) \right) \\
    &\wedge \square_{[0, T]}\left(\wedge_{i=1}^4 \neg \text{Obs}_i\right)\wedge \left( \Diamond_{[0, T]} \text{Goal} \right).
    \end{aligned} 
    \end{equation*}

    \item (Door puzzle) A robot must pick up several keys in specific regions $\text{K}_i$ before passing through the associated doors $\text{D}_i$ and finally reaching the goal $\text{G}$. Along the way, it must avoid the walls $\text{W}_i$. The specification can be written as 
    $$\varphi= \wedge_{i=1}^{4} (\neg \text{D}_i \cup_{[0, 2]} \text{K}_i) \wedge_{i=1}^5 \left( \square_{[0, T]} \neg \text{W}_i \right)\wedge \left( \Diamond_{[0, T]} \text{G} \right).$$
    
\end{itemize}

\subsection{Comparison with Other Algorithms}

Fig.~\ref{fig. planning} presents our planning results on the benchmarks. To evaluate the performance of our approach, we compare them with two other methods: a standard discrete-time MICP method \cite{belta2019formal} and a piece-wise linear (PWL) paths approach \cite{sun2022multi}. We report the runtimes, robustness, and maximum tracking errors as shown in Table \ref{table: comparision results}. To ensure a fair comparison, we use the minimum number of PWL paths with a consistent dynamic model and constraints. For the MICP method, we also set $\Delta t = 0.2 \text{s}$ to match the tracking frequency.

\subsubsection{Run time}

It can be observed that when the specifications have short time intervals for the sub-specifications, PWL may also require more segments and take more time than our method. Additionally, our approach requires fewer segments over longer time horizons, which achieves significant improvements in efficiency compared to standard MICP. Finally, although our method requires more continuous variables for generating Bézier curves, our time consumption did not significantly increase.

\subsubsection{Tracking errors}

To evaluate the dynamical feasibility of our generated trajectory, we employ the nonlinear discrete-time car-like vehicle model and the standard model predictive controller (MPC) \cite{borrelli2005mpc}. Our tracking error is smaller than that of PWL as our trajectory is $C^2$-continuous, whereas PWL is only $C^0$-continuous and aims to minimize travel time. This often results in dramatic changes in velocity between segments and causes the tracking error to increase at sharp turns as shown in Fig.~\ref{fig1(b)}.

\begin{figure*}[!tbp]
\vspace{3.mm}
    \centering
    \parbox{6.8in}{
    \subfloat[Reach-avoid.]{
    \begin{minipage}[b]{0.21\linewidth}
    \centering
        \includegraphics[width=\linewidth]{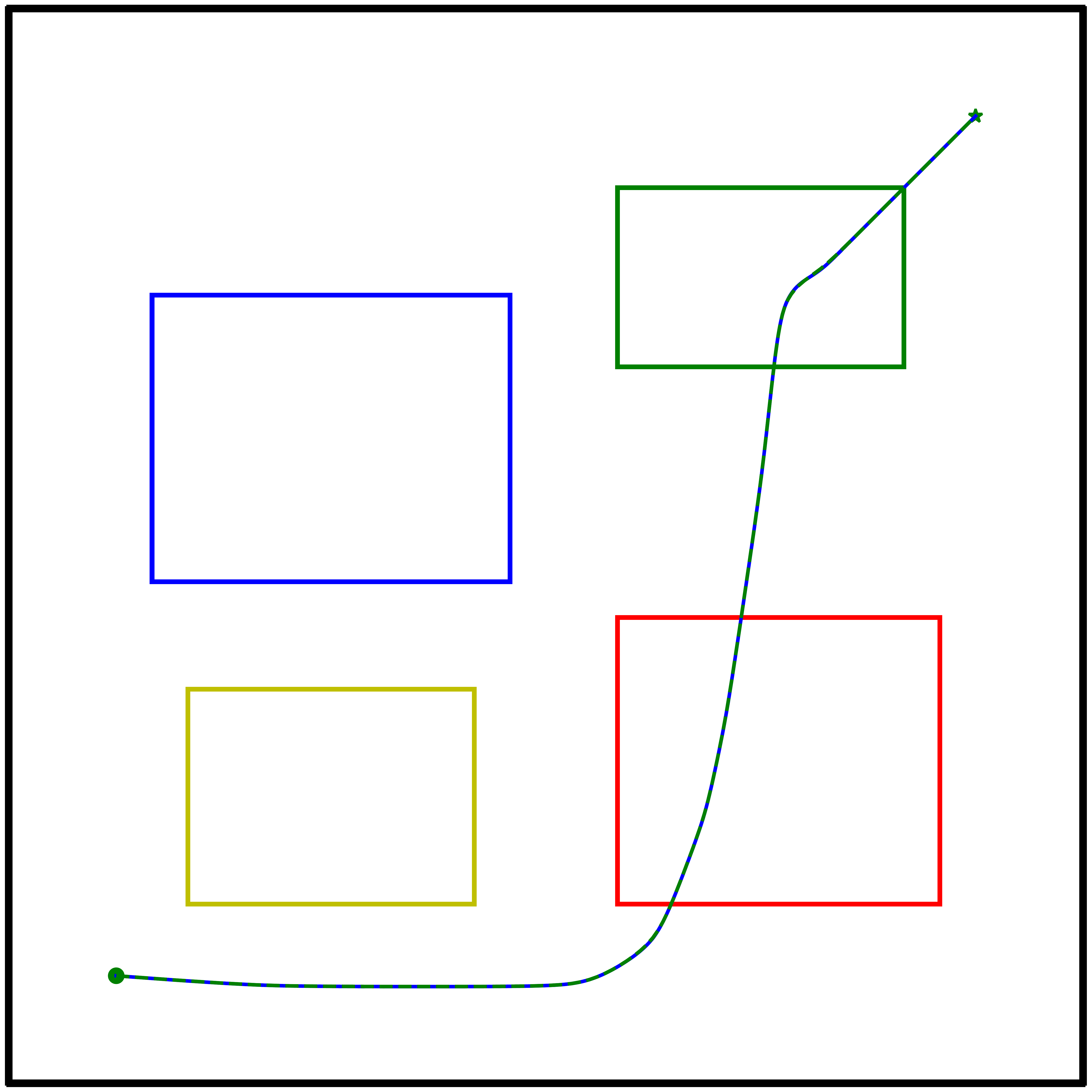}\label{fig:sub1a}\vspace{0.1mm}
        \includegraphics[width=\linewidth]{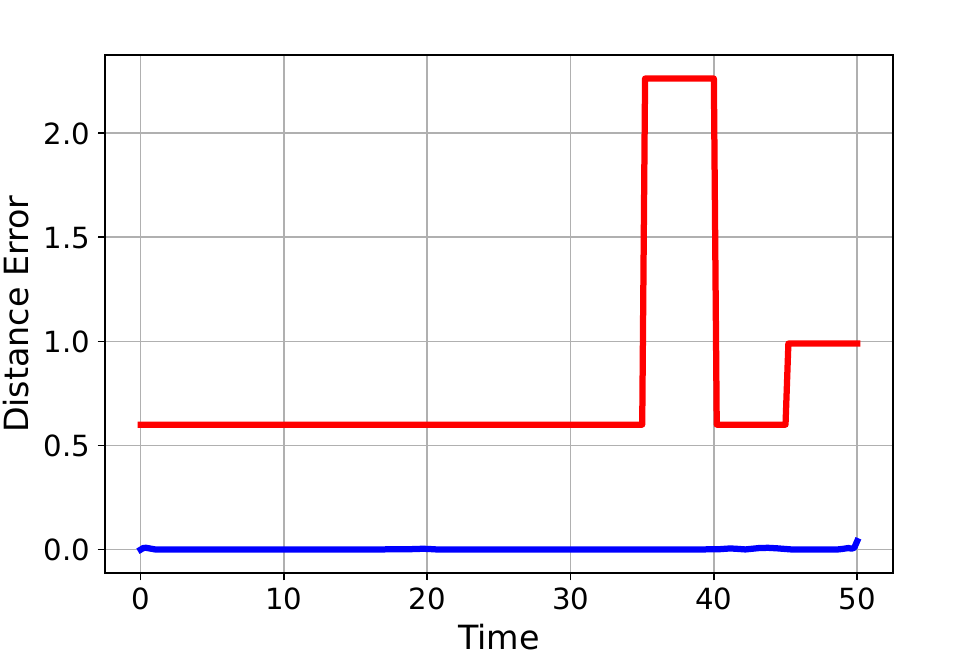}\label{fig:sub1b}
    \end{minipage}
    }    
    \subfloat[Complex reach-avoid.]{
    \begin{minipage}[b]{0.21\linewidth}
    \centering
        \includegraphics[width=\linewidth]{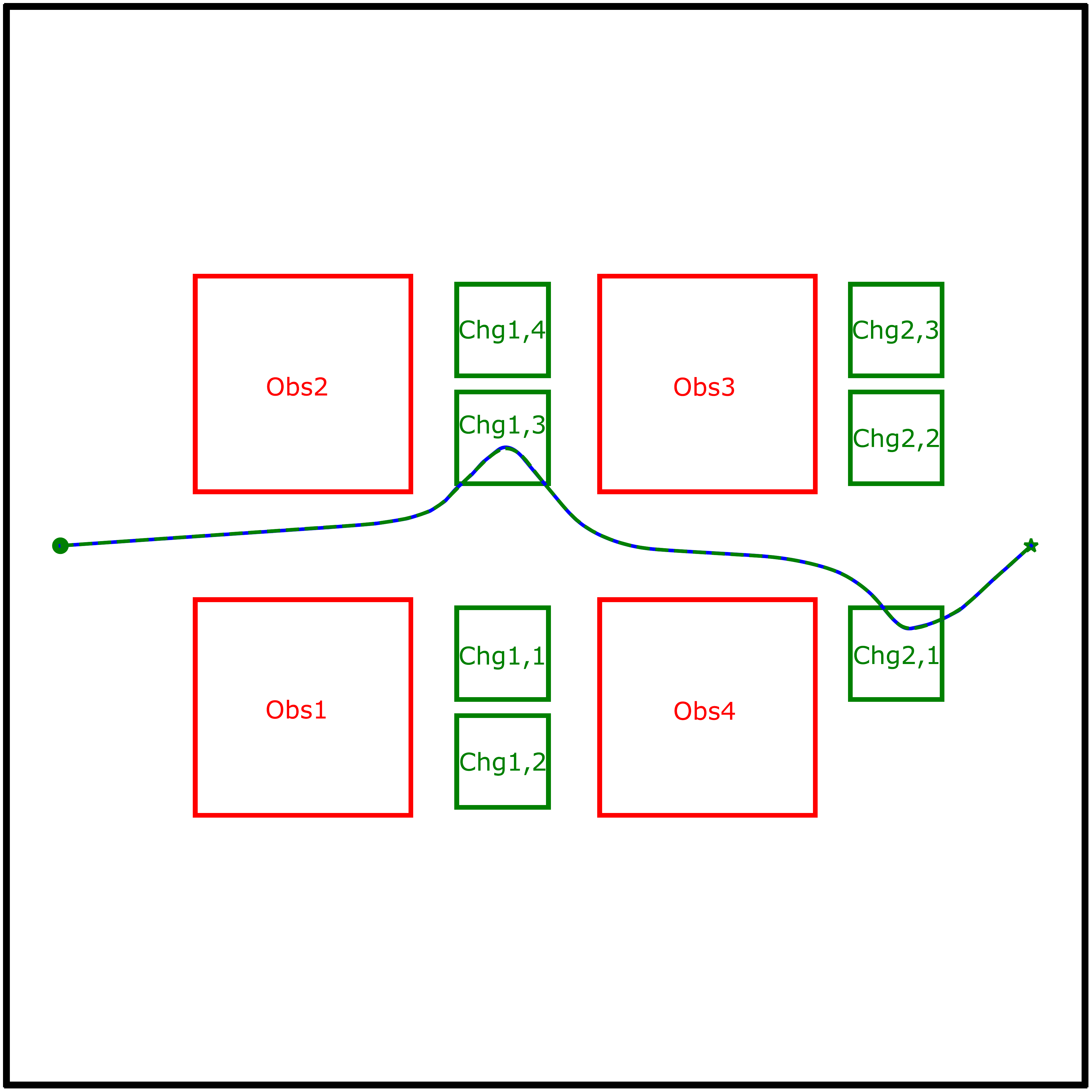}\label{fig:sub2a}\vspace{0.1mm}
        \includegraphics[width=\linewidth]{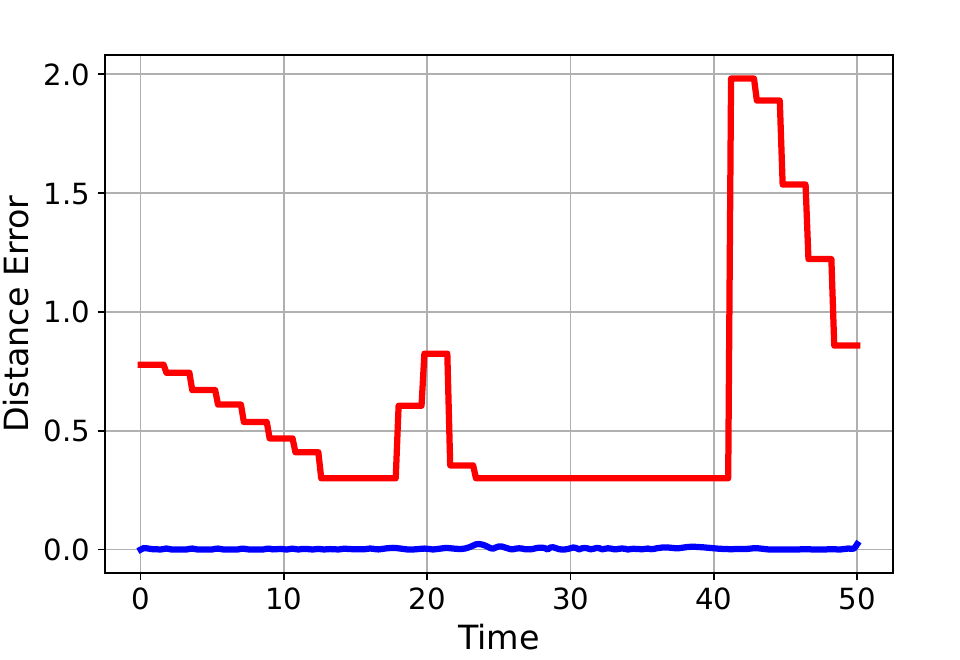}\label{fig:sub2b}
    \end{minipage}
    }
    \subfloat[Narrow passing.]{
    \begin{minipage}[b]{0.21\linewidth}
    \centering
        \includegraphics[width=\linewidth]{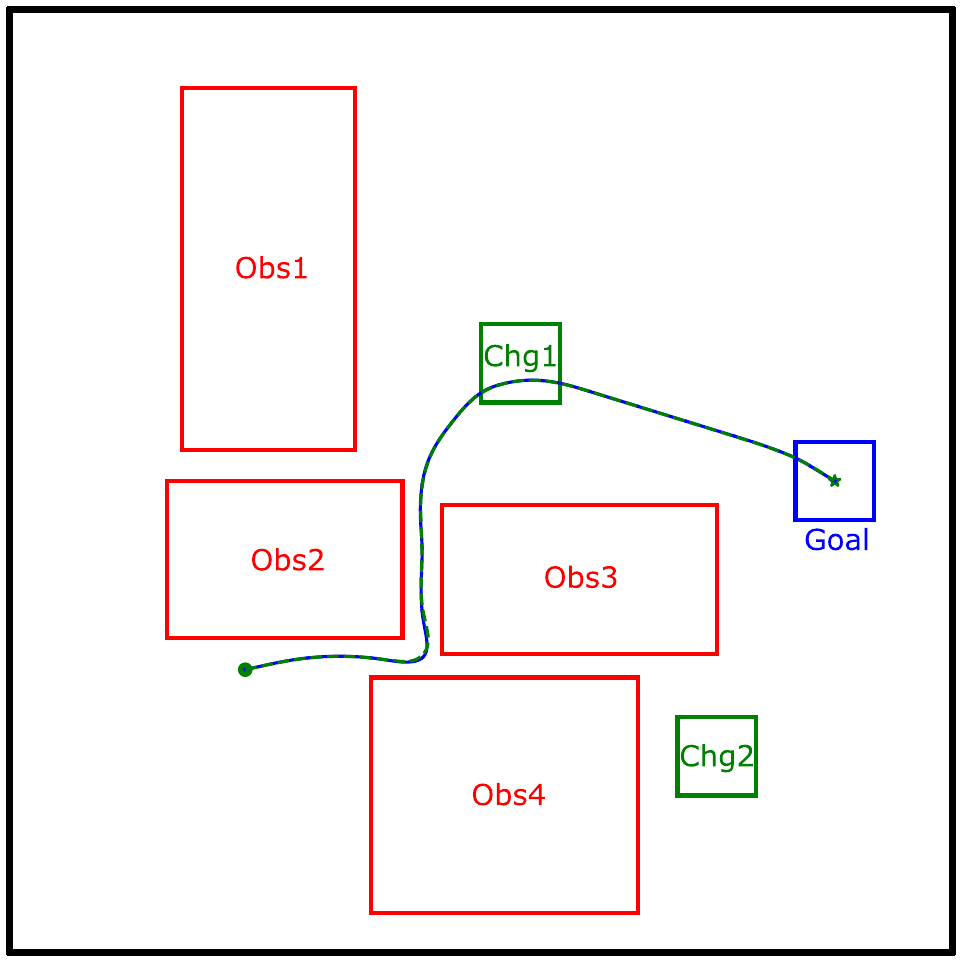}\label{fig:sub3a}\vspace{0.1mm}
        \includegraphics[width=\linewidth]{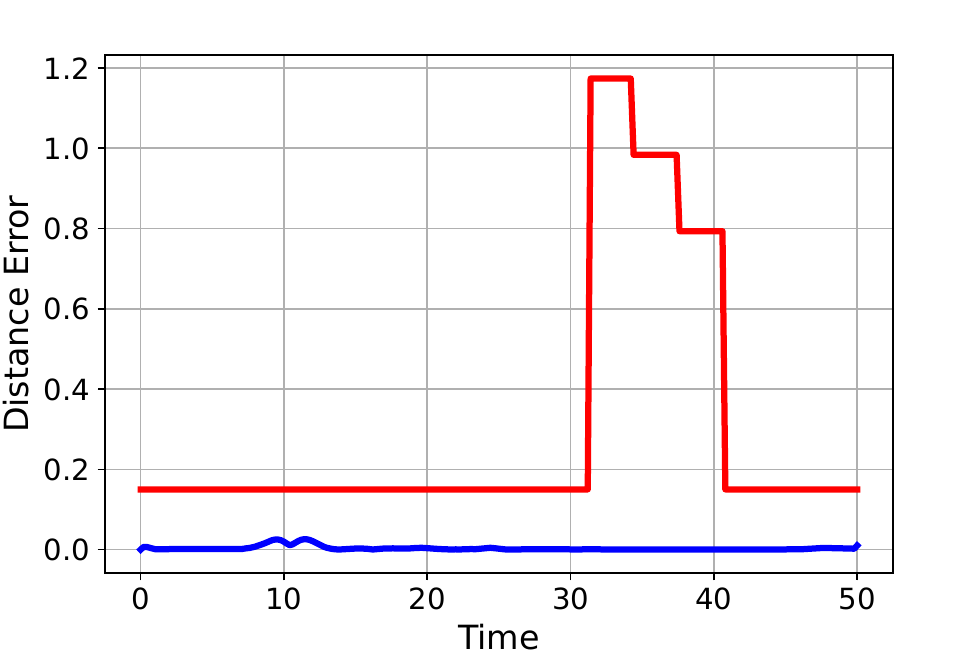}\label{fig:sub3b}
    \end{minipage}
    }
    \subfloat[Door puzzle.]{
    \begin{minipage}[b]{.311\linewidth}
    \centering
        \includegraphics[width=\linewidth]{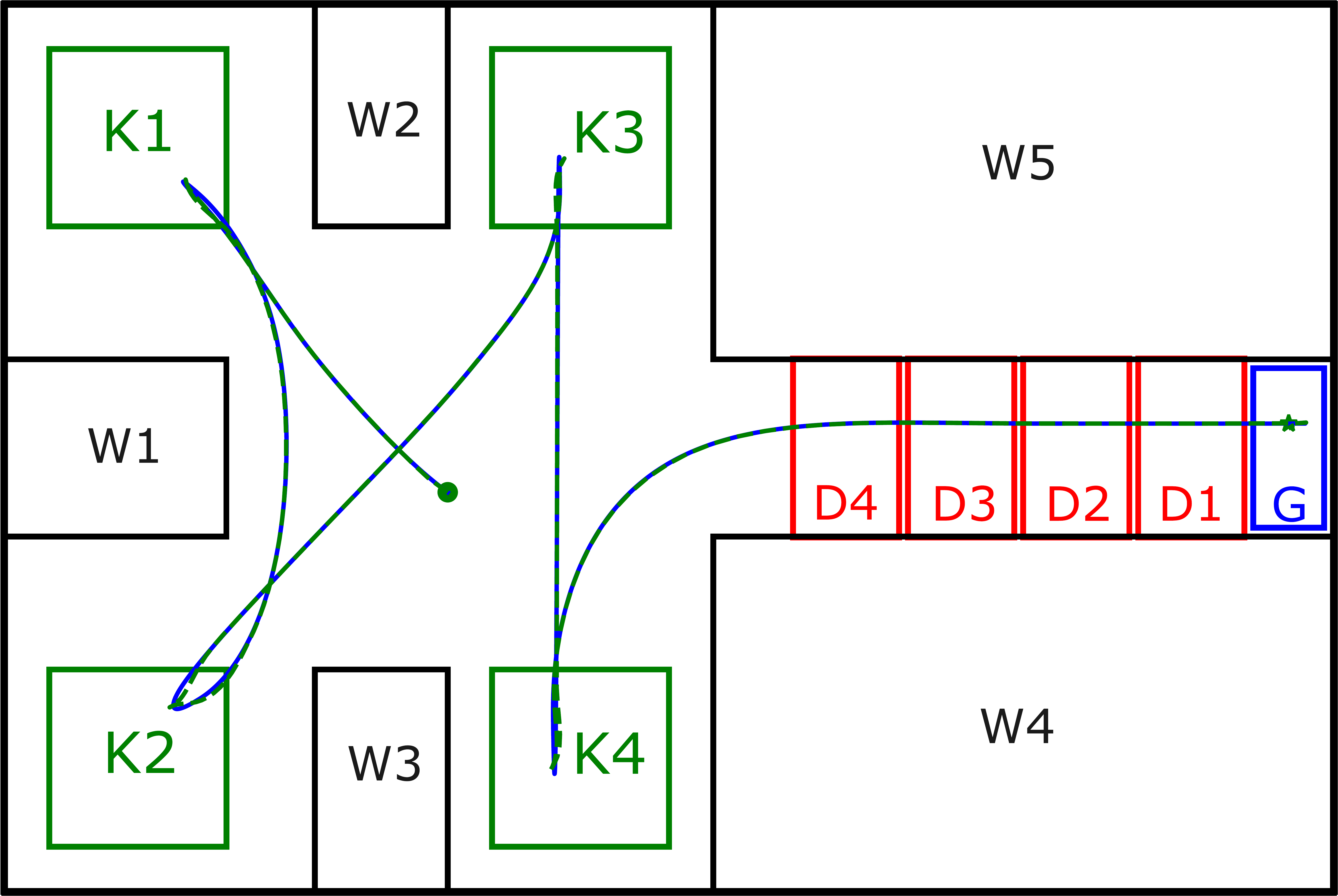}\label{fig:sub4a}\vspace{0.1mm}
        \includegraphics[width=.85\linewidth]{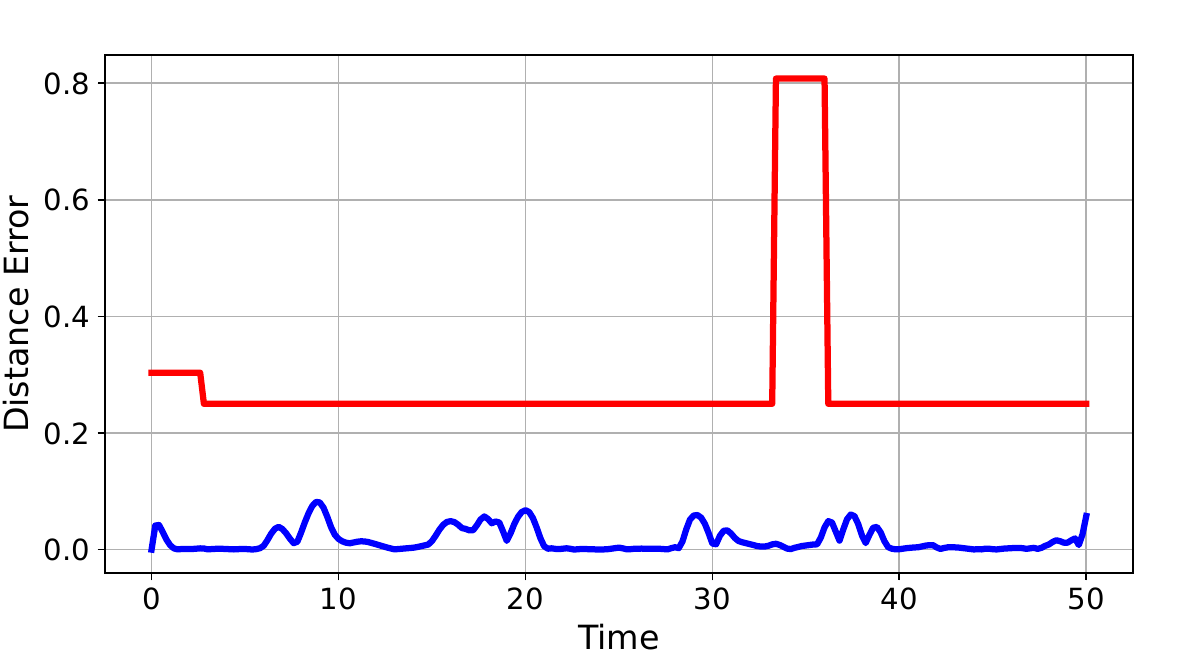}\label{fig:sub4b}
    \end{minipage}
    }
    }
\caption{Benchmarks and Results. The first row presents the planning outcomes, where solid lines depict the Bézier curves generated by the proposed method, and dashed lines represent the actual trajectory followed by the vehicle as it tracks these curves. Circles mark the starting points of the trajectories, while stars indicate the goal locations. The second row shows the tracking errors, with red lines illustrating the time-varying robustness and blue lines representing the distance tracking errors. }
\label{fig. planning}
\end{figure*}

\begin{table*}[!t]
\caption{Comparisons on benchmarks: runtimes, robustness, and maximum tracking errors.}
\label{table: comparision results}
\renewcommand{\arraystretch}{1.1}
\begin{center}
\begin{tabular}{|c||c|c|c|c|c|c|c|c|c|c|}
\hline 
\multirow{2}{*}{Mission} & \multirow{2}{*}{Time} &  \multicolumn{3}{c|}{Runtimes (s)} &
 \multicolumn{3}{c|}{Robustness} & \multicolumn{3}{c|}{Maximum tracking error}\\
\cline{3-11}
& Horizon (s) & Ours & \begin{tabular}{c} Standard  \\ MICP \end{tabular}& PWL & 
 \begin{tabular}{c} Ours \\ (min, max) \end{tabular} & \begin{tabular}{c} Standard  \\ MICP \end{tabular}& PWL & Ours & \begin{tabular}{c} Standard  \\ MICP \end{tabular} & PWL\\
\hline 
\multirow{2}{*}{Reach-avoid} 
& 15  & \textbf{0.14} & 5.12 & 0.32 & (0.60, 3.21) & 0.60 & 0.60 & 0.27 & \textbf{0.21} & 1.70\\
& 50 & \textbf{0.2}  & 423.63 & 0.38 &(0.60, 2.26)& 0.60 & 0.60 & \textbf{0.04} & 0.03 & 1.76\\
\hline
\multirow{2}{*}{Complex Reach-avoid} 
& 20 & 1.1 & 9.74 & \textbf{0.50} & (0.30, 0.53) & 0.30 & 0.30 & \textbf{0.12} & 0.23 & 1.36\\
& 50 & \textbf{1.04} & 8.74 & 3.98 & (0.30, 1.98) & 0.30 & 0.30 & \textbf{0.02} & 0.24 & 1.35\\
\hline
\multirow{2}{*}{Narrow passing} 
& 15 & 1.8 & 8.66 & \textbf{0.31} & (0.15, 1.37) & 0.15 & 0.15 & \textbf{0.05} & 0.06 & 0.27\\
& 50 & \textbf{0.46} & 155.17 & 0.46 & (0.15, 1.17) & 0.15 & 0.15 & \textbf{0.02} & 0.13 & 0.35 \\
\hline
\multirow{2}{*}{Door puzzle} 
& 25 & 13.06 & $>$2500 & \textbf{4.97} & (0.25, 0.85) & - & 0.25 & \textbf{0.2} & - & 1.06 \\

& 50 & \textbf{5} & $>$2500 & 5.64 & (0.25, 0.8) & - & 0.25 & \textbf{0.08} & - & 1.82 \\
\hline
\end{tabular}
\end{center}
\end{table*}

\section{Conclusion}
In this letter, we propose a method that can efficiently generate Bézier curves to satisfy the STL specification with time-varying robustness. We prove that our robustness measure satisfies the soundness property. The experimental results demonstrate that our generated trajectory is smooth and can be tracked with minimal tracking errors. In future research we will address scalability challenges posed by increasingly complex specifications and higher-dimensional problems.










\bibliographystyle{unsrt}
\bibliography{LCSS}

\begin{thebibliography}{10}

\bibitem{pant2018fly}
Yash~Vardhan Pant, Houssam Abbas, Rhudii~A Quaye, and Rahul Mangharam.
\newblock Fly-by-logic: Control of multi-drone fleets with temporal logic objectives.
\newblock In {\em 2018 ACM/IEEE 9th International Conference on Cyber-Physical Systems (ICCPS)}, pages 186--197. IEEE, 2018.

\bibitem{kurtz2020trajectory}
Vince Kurtz and Hai Lin.
\newblock Trajectory optimization for high-dimensional nonlinear systems under stl specifications.
\newblock {\em IEEE Control Systems Letters}, 5(4):1429--1434, 2020.

\bibitem{fainekos2009robustness}
Georgios~E Fainekos and George~J Pappas.
\newblock Robustness of temporal logic specifications for continuous-time signals.
\newblock {\em Theoretical Computer Science}, 410(42):4262--4291, 2009.

\bibitem{yang2020continuous}
Guang Yang, Calin Belta, and Roberto Tron.
\newblock Continuous-time signal temporal logic planning with control barrier functions.
\newblock In {\em 2020 American Control Conference (ACC)}, pages 4612--4618. IEEE, 2020.

\bibitem{raman2014model}
Vasumathi Raman, Alexandre Donzé, Mehdi Maasoumy, Richard~M. Murray, Alberto Sangiovanni-Vincentelli, and Sanjit~A. Seshia.
\newblock Model predictive control with signal temporal logic specifications.
\newblock In {\em 53rd IEEE Conference on Decision and Control}, pages 81--87, 2014.

\bibitem{belta2019formal}
Calin Belta and Sadra Sadraddini.
\newblock Formal methods for control synthesis: An optimization perspective.
\newblock {\em Annual Review of Control, Robotics, and Autonomous Systems}, 2:115--140, 2019.

\bibitem{kurtz2022mixed}
Vincent Kurtz and Hai Lin.
\newblock Mixed-integer programming for signal temporal logic with fewer binary variables.
\newblock {\em IEEE Control Systems Letters}, 6:2635--2640, 2022.

\bibitem{pant2017smooth}
Yash~Vardhan Pant, Houssam Abbas, and Rahul Mangharam.
\newblock Smooth operator: Control using the smooth robustness of temporal logic.
\newblock In {\em 2017 IEEE Conference on Control Technology and Applications (CCTA)}, pages 1235--1240. IEEE, 2017.

\bibitem{mehdipour2019arithmetic}
Noushin Mehdipour, Cristian-Ioan Vasile, and Calin Belta.
\newblock Arithmetic-geometric mean robustness for control from signal temporal logic specifications.
\newblock In {\em 2019 American Control Conference (ACC)}, pages 1690--1695. IEEE, 2019.

\bibitem{gilpin2020smooth}
Yann Gilpin, Vince Kurtz, and Hai Lin.
\newblock A smooth robustness measure of signal temporal logic for symbolic control.
\newblock {\em IEEE Control Systems Letters}, 5(1):241--246, 2021.

\bibitem{sun2022multi}
Dawei Sun, Jingkai Chen, Sayan Mitra, and Chuchu Fan.
\newblock Multi-agent motion planning from signal temporal logic specifications.
\newblock {\em IEEE Robotics and Automation Letters}, 7(2):3451--3458, 2022.

\bibitem{kurtz2023temporal}
Vince Kurtz and Hai Lin.
\newblock Temporal logic motion planning with convex optimization via graphs of convex sets.
\newblock {\em IEEE Transactions on Robotics}, 2023.

\bibitem{yang2024signal}
Tiange Yang, Yuanyuan Zou, Shaoyuan Li, Xiang Yin, and Tianyu Jia.
\newblock Signal temporal logic synthesis under model predictive control: A low complexity approach.
\newblock {\em Control Engineering Practice}, 143:105782, 2024.

\bibitem{baier2008principles}
Christel Baier and Joost-Pieter Katoen.
\newblock {\em Principles of model checking}.
\newblock MIT press, 2008.

\bibitem{darwiche2001decomposable}
Adnan Darwiche.
\newblock Decomposable negation normal form.
\newblock {\em Journal of the ACM (JACM)}, 48(4):608--647, 2001.

\bibitem{LaValle2006-oe}
Steven~M LaValle.
\newblock {\em Planning Algorithms}.
\newblock Cambridge University Press, May 2006.

\bibitem{donze2010robust}
Alexandre Donz{\'e} and Oded Maler.
\newblock Robust satisfaction of temporal logic over real-valued signals.
\newblock In {\em Formal Modeling and Analysis of Timed Systems: 8th International Conference, FORMATS 2010, Klosterneuburg, Austria, September 8-10, 2010. Proceedings 8}, pages 92--106. Springer, 2010.

\bibitem{gurobi}
{Gurobi Optimization, LLC}.
\newblock {Gurobi Optimizer Reference Manual}, 2024.

\bibitem{cvxpy}
Steven Diamond, Eric Chu, and Stephen Boyd.
\newblock {CVXPY}: A {P}ython-embedded modeling language for convex optimization, version 0.2.
\newblock \url{http://cvxpy.org/}, May 2014.

\bibitem{borrelli2005mpc}
Francesco Borrelli, Paolo Falcone, Tamas Keviczky, Jahan Asgari, and Davor Hrovat.
\newblock Mpc-based approach to active steering for autonomous vehicle systems.
\newblock {\em International journal of vehicle autonomous systems}, 3(2-4):265--291, 2005.

\end{thebibliography}

\end{document}